\documentclass[twoside,onecolumn]{article} 
\usepackage[a4paper, margin=1in]{geometry} 
\pdfoutput=1  
\usepackage{cite}  
\usepackage{amsmath,amssymb,amsfonts}  
\usepackage[utf8]{inputenc}  
\usepackage{graphicx}  
\usepackage{textcomp}  
\usepackage{hyperref}  
\usepackage{booktabs}  
\usepackage{bm}  
\usepackage{capt-of}  
\usepackage{mathtools}  
\usepackage{siunitx}  
\usepackage[figurename=Figure]{caption}  
\usepackage{scrextend}  
\usepackage{xspace}  
\usepackage[table]{xcolor}  
\usepackage{multirow}  
\usepackage{wrapfig}

\usepackage{amsthm}  
\usepackage{thmtools}  
\usepackage{algorithm}  
\usepackage{algpseudocode}  
\usepackage[export]{adjustbox}  
\usepackage{url}  
\usepackage{subfig}  
\usepackage{lineno}  
\usepackage{pifont}  
\usepackage{cite}  
\usepackage[normalem]{ulem}  
\usepackage{epstopdf}  
\usepackage{tikz}  
\usetikzlibrary{shapes,arrows,positioning,calc,fit,backgrounds}  
\usepackage[a4paper, margin=1in]{geometry}  
\usepackage{amsmath} 
\usepackage{amssymb} 
\usepackage{xcolor} 

\newtheorem{theorem}{Theorem}[section] 
\newtheorem{lemma}[theorem]{Lemma} 
\newtheorem{corollary}[theorem]{Corollary} 
\newtheorem{remark}[theorem]{Remark} 
\newtheorem{assumption}[theorem]{Assumption} 
\newtheorem{definition}[theorem]{Definition} 

\newcommand{\ie}{i.e., }  

\definecolor{subsectioncolor}{rgb}{0,0.541,0.855}  

\DeclareMathOperator*{\argmax}{arg\,max}  

\newenvironment{keywords}  
  {\begin{quote}\small\textbf{Keywords:}\ }  
  {\end{quote}}  

\title{Addressing Domain Shift via Imbalance-Aware Domain Adaptation in Embryo Development Assessment}  

\author{Lei~Li\thanks{Lei Li is with University of Copenhagen, Copenhagen, Denmark, and also with University of Washington, Seattle, WA, USA (e-mail: lilei@di.ku.dk).}  
        \and Xinglin~Zhang\thanks{Xinglin Zhang and Jun Liang are with Shanghai Medical Image Insights Intelligent Technology Co., Ltd., Shanghai 200032, China (e-mails: xinglin.zhang@imagecore.com.cn; jun.liang@imagecore.com.cn).}  
        \and Jun~Liang\footnotemark[1]  
        \and Tao~Chen\thanks{Tao Chen is with University of Waterloo, Waterloo, ON, Canada (e-mail: t66chen@uwaterloo.ca).}  
        \thanks{Lei Li and Xinglin Zhang contributed equally to this work.}  
        \thanks{Corresponding author: Tao Chen (e-mail: t66chen@uwaterloo.ca).}  
} 

\begin{document} 

\maketitle  

Deep learning models in medical imaging face dual challenges: domain shift, where models perform poorly when deployed in settings different from their training environment, and class imbalance, where certain disease conditions are naturally underrepresented. We present Imbalance-Aware Domain Adaptation (IADA), a novel framework that simultaneously tackles both challenges through three key components: (1) adaptive feature learning with class-specific attention mechanisms, (2) balanced domain alignment with dynamic weighting, and (3) adaptive threshold optimization. Our theoretical analysis establishes convergence guarantees and complexity bounds. Through extensive experiments on embryo development assessment across four imaging modalities, IADA demonstrates significant improvements over existing methods, achieving up to 25.19\% higher accuracy while maintaining balanced performance across classes. In challenging scenarios with low-quality imaging systems, IADA shows robust generalization with AUC improvements of up to 12.56\%. These results demonstrate IADA's potential for developing reliable and equitable medical imaging systems for diverse clinical settings. The code is made public available at \url{https://github.com/yinghemedical/imbalance-aware_domain_adaptation}.



 
\section{Introduction}

Deep learning has revolutionized medical image analysis, demonstrating remarkable potential for automating complex diagnostic tasks~\cite{litjens2017survey}. However, two critical challenges emerge when deploying these systems in real-world clinical settings: domain shift and class imbalance. While these challenges have been studied independently~\cite{zhou2006training, ganin2016domain}, their interaction - particularly in medical imaging contexts - remains understudied and poses significant barriers to clinical adoption.

Domain shift occurs when deep learning models trained on data from one medical context show degraded performance when deployed in different settings~\cite{kanakasabapathy2021adaptive}. This challenge is particularly evident in embryo development assessment, where models trained on high-end clinical time-lapse imaging systems often perform poorly when applied to data from portable microscopes or smartphone-based systems~\cite{dimitriadis2022artificial}. As illustrated in Fig.~\ref{fig:teasing}, the transition from high-end to low-end imaging systems introduces significant variations in image quality, complicating the model's ability to maintain consistent performance.

\begin{figure}[!t]
	\centering
    \includegraphics[width=.5\textwidth]{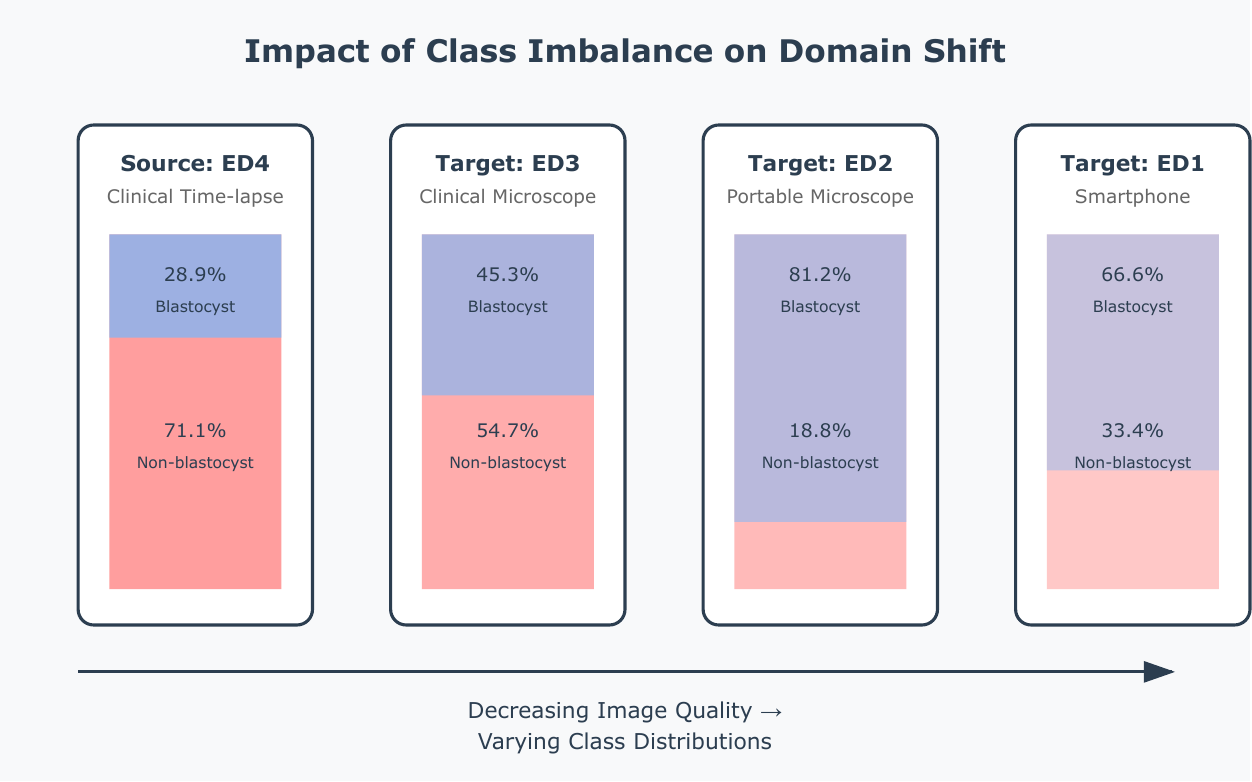} 
	\caption{\label{fig:teasing} Illustration of the impact of class imbalance on domain shift in embryo classification. The figure shows the transition from high-end to low-end imaging systems and their varying class distributions. Blastocyst refers to a critical stage in early embryo development (typically 5-6 days after fertilization) characterized by the formation of a fluid-filled cavity, while non-blastocyst encompasses earlier developmental stages. The decreasing opacity from left to right represents the degradation in image quality across domains. }
\end{figure}

This challenge is compounded by inherent class imbalance in medical data~\cite{chawla2002smote,li2024medflip}. In embryo assessment, the distribution of developmental stages varies naturally and differs significantly across imaging modalities~\cite{rubio2014clinical}. For instance, our analysis reveals that blastocyst-stage embryos represent 28.9\% of samples in clinical time-lapse systems (ED4) but comprise 81.2\% in portable microscope data (ED2). This disparity reflects both biological realities and differences in clinical protocols across settings~\cite{zaninovic2020artificial}.

The intersection of these challenges creates a complex problem. Existing approaches have addressed them separately - domain adaptation techniques focus on bridging domain gaps without considering class distributions~\cite{tzeng2017adversarial,li2023segment}, while methods for handling imbalanced data typically assume consistent domain characteristics~\cite{he2008adasyn}. However, as Fig.~\ref{fig:teasing} demonstrates, in real-world medical applications~\cite{wang2024mamba}, these challenges are inherently intertwined. The varying proportions of blastocyst versus non-blastocyst stages across different imaging modalities suggest that effective solutions must simultaneously address both domain shift and class imbalance.

In this paper, we present Imbalance-Aware Domain Adaptation (IADA), a novel framework that simultaneously tackles both challenges through three key innovations: (1) adaptive feature learning with class-specific attention mechanisms, (2) balanced domain alignment with dynamic weighting, and (3) adaptive threshold optimization. Through theoretical analysis and extensive experimentation on embryo development assessment across four distinct imaging modalities, we demonstrate IADA's effectiveness in maintaining robust performance despite varying image quality and class distributions. Our results show up to 25.19\% improvement in accuracy over existing methods while maintaining balanced performance across classes, suggesting IADA's potential for developing more reliable and equitable medical imaging systems for diverse clinical settings~\cite{salih2023embryo}.

\section{Related Work}

\subsection{Domain Shift}
Domain shift arises when there is a discrepancy between the distribution of training and testing datasets, leading to a decline in model performance. This phenomenon has been widely studied across multiple domains, including natural images, speech processing, and medical diagnostics \cite{quinonero2008dataset, torralba2011unbiased, tzeng2017adversarial}. 

Domain shift can occur due to various factors such as differences in imaging devices, acquisition protocols, environmental conditions, or population demographics. Techniques to combat domain shift include domain adaptation, domain generalization, and adversarial training. These approaches seek to align feature distributions across domains or make models robust to distributional variations \cite{ganin2016domain, tzeng2017adversarial}. For instance, adversarial learning frameworks such as Gradient Reversal Layer (GRL) have demonstrated success in reducing the discrepancy between source and target domains \cite{ganin2016domain}.

Domain shift is especially problematic in medical image analysis due to the heterogeneity of medical data. Variations in imaging modalities (e.g., MRI, CT, ultrasound), equipment manufacturers, and clinical protocols exacerbate the issue. Moreover, medical datasets are often small and suffer from class imbalances, making models more susceptible to overfitting and domain bias \cite{litjens2017survey, cheplygina2019not}. In \cite{kanakasabapathy2021adaptive}, Kanakasabapathy et al. develop Adaptive Adversarial Neural Networks to address domain shift in lossy and domain-shifted datasets of medical images. Their work highlights the importance of domain alignment techniques in preserving diagnostic accuracy. Other studies have explored self-supervised learning, meta-learning, and generative adversarial networks (GANs) to mitigate domain discrepancies \cite{liu2020shape, dou2018unsupervised, arani2022multi}.

Embryo research stands out as a particularly critical area of medical imaging due to its implications for assisted reproductive technologies (ART). Accurate assessment of embryo quality is essential for optimizing implantation success rates and improving patient outcomes. However, this field faces unique challenges due to the sensitivity of embryo imaging, variations in clinical settings, and limited availability of annotated datasets \cite{vernaeve2008clinical, gad2020deep, cruz2022automatic}.

The focus on embryo research is motivated by its potential to revolutionize ART practices through advanced imaging and machine learning. By addressing domain shift, researchers aim to create robust algorithms capable of generalizing across diverse patient populations and clinical environments \cite{kanakasabapathy2021adaptive, desai2021computational}. This involves leveraging techniques such as multi-domain learning, synthetic data generation, and adaptive adversarial training to ensure reliable performance in diverse scenarios \cite{xu2021synthetic, peng2020domain, zhu2017unpaired,li2024cpseg}. In this work, we follow the experimental protocol in the state-of-the-art research \cite{kanakasabapathy2021adaptive}, as a foundational framework to study and address the challenges posed by domain shift in medical image analysis.

\subsection{Learning with Imbalanced Data}
Handling imbalanced data is a persistent challenge in medical applications, where minority classes often represent critical conditions requiring accurate predictions. Data-level methods aim to address imbalance by modifying the dataset. Oversampling techniques, such as Synthetic Minority Over-sampling Technique (SMOTE) \cite{chawla2002smote}, generate synthetic minority class samples to balance the class distribution. Adaptive Synthetic Sampling (ADASYN) \cite{he2008adasyn} extends SMOTE by focusing on hard-to-learn instances, improving model performance on challenging samples. Alternatively, undersampling techniques, such as random undersampling and Cluster Centroids \cite{yen2009cluster}, reduce the majority class size but may risk losing valuable information.

Algorithmic solutions incorporate imbalance-handling mechanisms directly into the learning process. Cost-sensitive learning (CSL) \cite{zhou2006cost} assigns higher misclassification costs to minority class instances, guiding the model toward improved performance on these classes. Ensemble methods, such as boosting and bagging, are also effective for imbalanced data by combining diverse models to mitigate bias \cite{seiffert2010rusboost}.

Hybrid methods combine oversampling and undersampling techniques to refine the dataset further. For instance, SMOTE with Tomek links \cite{batista2004smotetomek} generates synthetic samples while removing overlapping instances, creating a more balanced and cleaner dataset.

Recent research has introduced advanced techniques for imbalanced data. Density-Aware Personalized Training \cite{huo2022density} decouples feature extraction and classification, using density-aware loss and learnable cost matrices. Federated Learning for Class-Imbalanced Medical Image Classification (FedIIC) \cite{wu2023fediic} addresses imbalance in decentralized settings, employing contrastive learning and dynamic margins. Progressive Class-Center Triplet Loss \cite{li2022progressive} uses a two-stage training approach to separate and compact class distributions.

In medical diagnostics, methods like SMOTE and cost-sensitive learning have been crucial for identifying rare diseases with low prevalence rates \cite{haixiang2017learning}. In medical imaging, approaches such as FedIIC have improved the classification of rare abnormalities \cite{wu2023fediic}.

\section{Problem Formulation}
Medical domain adaptation presents a dual challenge: addressing both domain shift and class imbalance simultaneously. In healthcare applications, models trained on data from one medical context (source domain) often show degraded performance when deployed in different settings (target domain), particularly when dealing with imbalanced disease distributions. This challenge is compounded by various factors including differences in patient populations, imaging equipment, clinical protocols, and institutional practices.

Let us first formalize the problem setup. We denote the source domain labeled dataset as $\mathcal{D}_s = \{(x_i^s, y_i^s)\}_{i=1}^{n_s}$, where:
\begin{itemize}
    \item $x_i^s \in \mathcal{X} \subset \mathbb{R}^d$ represents $d$-dimensional input features (e.g., medical images, clinical measurements)
    \item $y_i^s \in \mathcal{Y} = \{1, ..., C\}$ denotes the corresponding class labels (e.g., disease diagnoses)
    \item $n_s$ is the total number of source domain samples
\end{itemize}

Similarly, we define the target domain dataset as $\mathcal{D}_t = \{x_j^t\}_{j=1}^{n_t}$, which is typically unlabeled in real-world medical scenarios. The fundamental domain shift can manifest in three distinct ways:

\begin{itemize}
    \item \textbf{Covariate shift}: $P(X^s) \neq P(X^t)$, indicating differences in feature distributions
    \item \textbf{Label shift}: $P(Y^s) \neq P(Y^t)$, reflecting varying disease prevalence
    \item \textbf{Concept shift}: $P(Y|X^s) \neq P(Y|X^t)$, suggesting different feature-disease relationships
\end{itemize}

\section{Imbalance-Aware Domain Adaptation}
To address the challenges of medical domain shift while accounting for inherent class imbalances in medical data, we propose \textit{Imbalance-Aware Domain Adaptation} (IADA). Our framework integrates three key innovations: adaptive feature learning that captures class-specific characteristics, balanced domain alignment that ensures fair representation across classes, and dynamic threshold optimization that adapts to varying class distributions. By designing an end-to-end training process, our method simultaneously optimizes both domain adaptation and class balance objectives.

\subsection{Imbalance-Aware Feature}
At the core of our approach lies a carefully designed feature extractor $F_\theta: \mathcal{X} \rightarrow \mathcal{Z}$ that addresses class imbalance through a novel attention mechanism. This feature extraction process consists of two main stages that work in concert to produce robust, class-aware representations:

\begin{enumerate}
    \item \textbf{Base Feature Extraction}: 
    The first stage begins with extracting class-specific features. For each input sample $x_i$, we process it through a sophisticated pipeline:
    \begin{equation}
        f_c(x_i) = h_c(g(x_i))
    \end{equation}
    Here, $g(\cdot)$ serves as a shared backbone network that captures general features, while $h_c(\cdot)$ represents class-specific adaptation layers that fine-tune these features for each class's unique characteristics.

    \item \textbf{Attention Mechanism}:
    Building upon these base features, we introduce an attention mechanism that dynamically weights the importance of different class-specific features. The attention weights $\alpha_c(x_i)$ are computed through a softmax operation:
    \begin{equation}
    \label{eqn:attn}
        \alpha_c(x_i) = \frac{\exp(w_c^T g(x_i))}{\sum_{k=1}^C \exp(w_k^T g(x_i))}
    \end{equation}
    In this formulation, $w_c$ represents learnable attention vectors that help the model focus on the most relevant features for each class.
\end{enumerate}

To combine these components into a final representation, we compute a weighted sum of the class-specific features:
\begin{equation}
    z_i = F_\theta(x_i) = \sum_{c=1}^C \alpha_c(x_i) \cdot f_c(x_i)
\end{equation}

This carefully crafted architecture provides several key benefits:
\begin{itemize}
    \item It ensures features capture class-specific nuances through dedicated extractors
    \item It gives minority classes fair representation through targeted attention mechanisms
    \item It maintains flexibility to handle varying class distributions across different domains
\end{itemize}

\subsection{Adversarial Domain Alignment}
To bridge the gap between source and target domains while maintaining class balance, we employ an advanced adversarial framework with several key innovations:

\begin{enumerate}
    \item \textbf{Domain Discriminator}: 
    We implement a sophisticated discriminator $D_\phi: \mathcal{Z} \rightarrow [0, 1]$ as a multi-layer neural network. This discriminator incorporates three crucial components:
    \begin{itemize}
        \item A gradient reversal layer that enables adversarial training
        \item Class-balanced batch sampling to ensure fair representation
        \item Instance weighting that accounts for varying class frequencies
    \end{itemize}

    \item \textbf{Adversarial Loss}:
    To ensure effective domain alignment while respecting class balance, we define our adversarial objective as:
    \begin{align}
    \label{eqn:adv_loss}
        \mathcal{L}_{adv} = &\mathbb{E}_{x^s \sim \mathcal{D}_s}[\omega(y^s)\log D_\phi(F_\theta(x^s))] \\
        &+ \mathbb{E}_{x^t \sim \mathcal{D}_t}[\log(1 - D_\phi(F_\theta(x^t)))]
    \end{align}
    To address class imbalance, we introduce class-specific weights $\omega(y^s)$, defined as:
    \begin{equation}
        \omega(y^s) = \frac{1}{C\pi_{y^s}^s}
    \end{equation}
    where $C$ is the number of classes and $\pi_{y^s}^s$ represents the proportion of samples in class $y^s$ in the source domain.
\end{enumerate}

\subsection{Imbalance-Aware Classification}
The final component of our framework is an adaptive classification module $C_\psi: \mathcal{Z} \rightarrow \mathcal{Y}$ that dynamically adjusts to class imbalance through three mechanisms:

\begin{enumerate}
    \item \textbf{Adaptive Thresholds}:
    To account for varying class distributions, we compute class-specific thresholds that adapt to class frequencies:
    \begin{equation}
    \label{eqn:threshold}
        \tau_c = \beta \log(\frac{n_c^s}{\min_k n_k^s}) + \gamma
    \end{equation}
    Here, $\beta$ and $\gamma$ are learnable parameters that allow the thresholds to adapt during training, with $n_c^s$ representing the number of samples in class $c$ in the source domain.

    \item \textbf{Classification Decision}:
    Using these adaptive thresholds, we make the final classification decision through:
    \begin{equation}
        \hat{y} = \argmax_c \{C_\psi(z)_c - \tau_c\}
    \end{equation}
    This formulation ensures that minority classes receive fair consideration by adjusting decision boundaries based on class frequencies.

    \item \textbf{Confidence Calibration}:
    To ensure reliable probability estimates, we incorporate temperature scaling:
    \begin{equation}
        p(y|z) = \text{softmax}(C_\psi(z)/T)
    \end{equation}
    The temperature parameter $T$ is learned during training to optimize probability calibration.
\end{enumerate}

\subsection{Training Objective}
To bring all components together into a cohesive framework, we formulate a comprehensive training objective:

\begin{equation}
    \min_{\theta, \psi} \max_{\phi} \mathcal{L}_{cls}(\theta, \psi) - \lambda_{adv} \mathcal{L}_{adv}(\theta, \phi) + \lambda_{reg} \mathcal{R}(\theta, \psi)
    \label{eqn:objective}
\end{equation}

This objective consists of three carefully designed components:

\begin{enumerate}
    \item \textbf{Classification Loss} $\mathcal{L}_{cls}$:
    We employ a weighted focal loss to address class imbalance:
    \begin{equation}
        \mathcal{L}_{cls} = -\frac{1}{n_s}\sum_{i=1}^{n_s} \omega(y_i^s)(1-p_{y_i^s})^\gamma \log(p_{y_i^s})
    \end{equation}
    The focal loss term $(1-p_{y_i^s})^\gamma$ helps focus training on hard examples, while class weights $\omega(y_i^s)$ balance the contribution of different classes.

    \item \textbf{Regularization Term} $\mathcal{R}$:
    To prevent overfitting and ensure robust feature learning, we combine multiple regularization strategies:
    \begin{equation}
        \mathcal{R} = \lambda_1 \|\theta\|_2^2 + \lambda_2 \mathcal{L}_{cons} + \lambda_3 \mathcal{L}_{div}
    \end{equation}
    This includes L2 regularization ($\|\theta\|_2^2$), consistency regularization ($\mathcal{L}_{cons}$) across augmented samples, and feature diversity promotion ($\mathcal{L}_{div}$).

    \item \textbf{Adversarial Term}:
    To ensure stable training, we implement a warming-up schedule for the adversarial weight:
    \begin{equation}\label{eqn:lambda}
        \lambda_{adv} = \lambda_0 \cdot \min(1, \frac{t}{\tau})
    \end{equation}
    This gradual increase in adversarial strength, controlled by the current iteration $t$ and warming-up period $\tau$, allows the model to first learn good features before focusing on domain alignment.
\end{enumerate}

\section{Theoretical Analysis}

In this section, we provide theoretical analyses, including generalization, convergence rate, and algorithmic complexity. First, let's establish some key assumptions and definitions:

\begin{theorem}[Generalization Bound with Class Imbalance]\label{thm:generalization}
Let $h \in \mathcal{H}$ be a hypothesis with expected errors $\epsilon_s(h)$ and $\epsilon_t(h)$ on the source and target domains respectively. For any $\delta > 0$, with probability at least $1-\delta$, the following bound holds:
$$\epsilon_t(h) \leq \epsilon_s(h) + \sum_{i=1}^C |\pi^s_i - \pi^t_i| + \sum_{i=1}^C \min(\pi^s_i, \pi^t_i)d_i(\mathcal{H}) + \lambda$$
where $\lambda$ represents the combined error of the ideal joint hypothesis.
\end{theorem}

\begin{proof}
The proof follows a structured approach through the decomposition and bounding of error terms. We begin by expressing the target error using class-conditional distributions: $\epsilon_t(h) = \sum_{i=1}^C \pi^t_i \epsilon_{t,i}(h)$, where $\epsilon_{t,i}(h)$ represents the error for class $i$ in the target domain. Next, we establish that the difference between source and target errors for each class is bounded by the domain discrepancy: $|\epsilon_{t,i}(h) - \epsilon_{s,i}(h)| \leq d_i(\mathcal{H})$.

Applying the triangle inequality and leveraging the class proportions, we can derive:
\begin{align*}
|\epsilon_t(h) - \epsilon_s(h)| &\leq \sum_{i=1}^C |\pi^t_i\epsilon_{t,i}(h) - \pi^s_i\epsilon_{s,i}(h)| \\
&\leq \sum_{i=1}^C |\pi^t_i - \pi^s_i|\epsilon_{t,i}(h) + \sum_{i=1}^C \min(\pi^s_i, \pi^t_i)d_i(\mathcal{H})
\end{align*}

The final bound is obtained by incorporating the ideal joint hypothesis error $\lambda$.
\end{proof}

\begin{remark}
Theorem \ref{thm:generalization} provides several key insights regarding domain adaptation under class imbalance. First, the bound explicitly captures how differences in class proportions between domains affect generalization through the term $\sum_{i=1}^C |\pi^s_i - \pi^t_i|$. Second, the domain discrepancy's impact is modulated by the minimum class proportion through $\min(\pi^s_i, \pi^t_i)$, suggesting that rare classes have less influence on the overall bound. Third, the additive nature of the bound indicates that both class imbalance and domain shift contribute independently to the generalization gap.
\end{remark}

\begin{corollary}[Balanced Domain Case]\label{thm:balanced}
In the special case where domains are perfectly balanced (i.e., $\pi^s_i = \pi^t_i$ for all $i$), the generalization bound simplifies to:
$$\epsilon_t(h) \leq \epsilon_s(h) + \sum_{i=1}^C \pi^s_i d_i(\mathcal{H}) + \lambda$$
\end{corollary}

\begin{remark}
Corollary \ref{thm:balanced} demonstrates the elegance of the bound under balanced conditions. The removal of the class proportion difference term $\sum_{i=1}^C |\pi^s_i - \pi^t_i|$ reflects the simplified learning scenario when source and target domains share identical class distributions. Moreover, the weighting of domain discrepancies by class proportions $\pi^s_i$ in the simplified bound suggests that even in balanced scenarios, the impact of domain shift remains class-dependent. This provides theoretical justification for maintaining class-specific adaptation mechanisms even when domains are balanced.
\end{remark}

The generalization bound explicitly depends on the difference in class proportions between domains, while the impact of domain shift is weighted by the minimum class proportion across domains. In balanced domains, the bound elegantly simplifies to a weighted sum of class-conditional discrepancies. Notably, classes with larger proportion differences contribute more significantly to the domain gap.

\begin{assumption}[Smoothness and Convexity]\label{assm:smoothness}
The loss function $\mathcal{L}$ is assumed to be $\beta$-smooth and $\mu$-strongly convex. Furthermore, the gradients are bounded such that $\|\nabla \mathcal{L}(w)\| \leq G$ for some constant $G$. Additionally, the class proportions are constrained to satisfy $\pi^s_i, \pi^t_i \in (0,1)$ with the normalization condition $\sum_{i=1}^C \pi^s_i = \sum_{i=1}^C \pi^t_i = 1$.
\end{assumption}
The assumption \ref{assm:smoothness} is standard in optimization theory and ensures the convergence of gradient-based methods. The smoothness and strong convexity conditions provide upper and lower quadratic bounds on the loss function, while the gradient bound prevents excessive parameter updates. The class proportion constraints ensure proper probability distributions across domains.

\begin{lemma}[Class-weighted Gradient Bound]
For any iteration $t$, the expected gradient norm satisfies:
$$E[\|\nabla \mathcal{L}_t(w_t)\|^2] \leq \sum_{i=1}^C \max(\pi^s_i, \pi^t_i)G^2$$
\end{lemma}

\begin{proof}
Using Jensen's inequality and the gradient bound:
\begin{align*}
E[\|\nabla \mathcal{L}_t(w_t)\|^2] &= E[\|\sum_{i=1}^C (\pi^s_i\nabla \mathcal{L}^s_i + \pi^t_i\nabla \mathcal{L}^t_i)\|^2] \\
&\leq \sum_{i=1}^C \max(\pi^s_i, \pi^t_i)E[\|\nabla \mathcal{L}^s_i\|^2 + \|\nabla \mathcal{L}^t_i\|^2] \\
&\leq \sum_{i=1}^C \max(\pi^s_i, \pi^t_i)G^2
\end{align*}
\end{proof}

\begin{theorem}[Convergence Rate with Class Imbalance]\label{thm:convergence}
Let $w_t$ be the parameters at iteration $t$ using learning rate $\eta_t = \frac{2}{\mu(t+\gamma)}$ where $\gamma = \max\{\frac{4\beta}{\mu}, 1\}$. Then:
$$E[\mathcal{L}(w_t) - \mathcal{L}(w^*)] \leq \frac{2\beta\Delta_0}{(\mu t + 4\beta)} + \frac{C_{\pi}G^2}{2\mu^2t}$$
where $\Delta_0 = \|w_0 - w^*\|^2$ and $C_{\pi} = \sum_{i=1}^C \max(\pi^s_i, \pi^t_i)$ is the class proportion factor.
\end{theorem}

\begin{proof}
We begin with the strong convexity condition:
$$\mathcal{L}(w_t) - \mathcal{L}(w^*) \leq \langle \nabla \mathcal{L}(w_t), w_t - w^* \rangle - \frac{\mu}{2}\|w_t - w^*\|^2$$

Using the update rule $w_{t+1} = w_t - \eta_t\nabla \mathcal{L}(w_t)$, we can derive:
\begin{align*}
\|w_{t+1} - w^*\|^2 &= \|w_t - \eta_t\nabla \mathcal{L}(w_t) - w^*\|^2 \\
&= \|w_t - w^*\|^2 - 2\eta_t\langle \nabla \mathcal{L}(w_t), w_t - w^* \rangle \\
& \quad + \eta_t^2\|\nabla \mathcal{L}(w_t)\|^2
\end{align*}

Taking expectation and applying Lemma 1:
$$E[\|w_{t+1} - w^*\|^2] \leq (1-\mu\eta_t)E[\|w_t - w^*\|^2] + \eta_t^2C_{\pi}G^2$$

With the chosen learning rate and telescoping the sum:
$$E[\|w_t - w^*\|^2] \leq \frac{4\Delta_0}{(\mu t + 4\beta)} + \frac{2C_{\pi}G^2}{\mu^2t}$$

The final result follows from strong convexity.
\end{proof}

\begin{remark}
Theorem \ref{thm:convergence} reveals that convergence is governed by two competing terms: a first-order term that depends on the initial distance to the optimum and decays as $O(1/t)$, and a second-order term affected by class proportions through $C_{\pi}$. This decomposition suggests that class imbalance primarily impacts the later stages of optimization when the second term becomes dominant.
\end{remark}

\begin{corollary}[Balanced Domain Convergence]\label{thm:balanced_convergence}
When the domains are balanced ($\pi^s_i = \pi^t_i$ for all $i$), the convergence rate simplifies to:
$$E[\mathcal{L}(w_t) - \mathcal{L}(w^*)] \leq \frac{2\beta\Delta_0}{(\mu t + 4\beta)} + \frac{G^2}{2\mu^2t}$$
\end{corollary}

\begin{remark}
Corollary \ref{thm:balanced_convergence} demonstrates that domain balance leads to optimal convergence rates. In this case, the class proportion factor $C_{\pi}$ reduces to unity, resulting in the standard convergence rate for strongly convex optimization. This suggests that maintaining balanced domains not only improves generalization but also accelerates optimization.
\end{remark}

The analysis suggests using adaptive learning rates that vary with class proportions, with larger rates for minority classes and smaller rates for majority classes to ensure stability. Furthermore, the results support the use of class-specific attention mechanisms, adaptive thresholds, and balanced batch sampling strategies to mitigate the impact of domain and class imbalance on convergence.


\begin{definition}[Class-specific Sample Sizes]
For source and target domains, we define the number of samples in class $i$ as $n^s_i = n_s\pi^s_i$ and $n^t_i = n_t\pi^t_i$ respectively, where $n_s$ represents the total number of source samples, $n_t$ represents the total number of target samples, and $\pi^s_i, \pi^t_i$ denote the corresponding class proportions.
\end{definition}


\begin{theorem}[Time Complexity]\label{thm:complexity}
The overall time complexity for one training epoch is:
$$T(n_s, n_t) = O(C(\max_i\{\pi^s_i, \pi^t_i\})(n_s + n_t)d + C^2\log C)$$
where $C$ represents the number of classes and $d$ represents the feature dimension.
\end{theorem}

\begin{proof}
The analysis encompasses multiple components of the algorithm. The feature extraction process requires $O(d)$ operations per sample, resulting in a total cost of $O((n_s + n_t)d)$. The class-specific attention mechanism involves computing attention weights at $O(Cd)$ per sample and performing weighted aggregation at $O(C)$ per sample, yielding a total cost of $O(C(n_s + n_t)d)$. 

The class-balanced batch sampling requires maintaining binary search trees for each class at $O(C\log(\max\{n^s_i, n^t_i\}))$ and performing per-class sampling at $O(\log(\max\{n^s_i, n^t_i\}))$, resulting in a total cost of $O(C\log(n_s\max_i\{\pi^s_i\}))$. 

Finally, adaptive threshold computation involves sorting class frequencies at $O(C\log C)$ and updating thresholds at $O(C)$, contributing $O(C\log C)$ to the total complexity. The summation of these components yields the stated complexity bound.
\end{proof}

\begin{remark}
Complexity analysis \ref{thm:complexity} reveals two major components: a linear term scaling with sample size and feature dimension, modulated by class imbalance through $\max_i\{\pi^s_i, \pi^t_i\}$, and a class-dependent term reflecting the overhead of maintaining class-specific structures. This decomposition highlights how class imbalance affects computational efficiency.
\end{remark}

\begin{lemma}[Space Complexity]
The space complexity of the algorithm is:
$$S(n_s, n_t) = O(\sum_{i=1}^C (\pi^s_i n_s + \pi^t_i n_t)d + C^2)$$
\end{lemma}

\begin{proof}
The space requirements arise from several components. Feature storage demands $O(\sum_{i=1}^C \pi^s_i n_s d)$ for the source domain and $O(\sum_{i=1}^C \pi^t_i n_t d)$ for the target domain. The attention mechanism requires $O(Cd)$ for weight matrices and $O(C)$ per sample for attention scores. Additional space is needed for class-specific data structures, including binary search trees at $O(C)$ and class statistics at $O(C^2)$. The combination of these requirements establishes the stated space complexity.
\end{proof}

\begin{theorem}[Computational Trade-offs]\label{thm:tradeoff}
Given a computational budget $B$, the optimal batch size $b_i$ for class $i$ is:
$$b_i = B\sqrt{\frac{\min(\pi^s_i, \pi^t_i)}{\sum_{j=1}^C \min(\pi^s_j, \pi^t_j)}}$$
\end{theorem}

\begin{proof}
The optimization problem is formulated as minimizing the sum of inverse batch sizes subject to a total budget constraint: $\min_{b_i} \sum_{i=1}^C \frac{1}{b_i}$ subject to $\sum_{i=1}^C b_i = B$. Using Lagrange multipliers, we form $\mathcal{L} = \sum_{i=1}^C \frac{1}{b_i} + \lambda(\sum_{i=1}^C b_i - B)$. Taking derivatives and solving the resulting system of equations yields the optimal batch size allocation.
\end{proof}

\begin{remark}
Theorem \ref{thm:tradeoff} establishes the optimal allocation of computational resources across classes. The square root dependence on class proportions represents a balance between processing efficiency and class representation, ensuring that minority classes receive sufficient attention while maintaining computational efficiency.
\end{remark}

\begin{corollary}[Balanced Case Complexity]
When domains achieve perfect balance with $\pi^s_i = \pi^t_i = \frac{1}{C}$, the time complexity reduces to:
$$T_{balanced}(n_s, n_t) = O(\frac{n_s + n_t}{C}d + C\log C)$$
\end{corollary}

\begin{remark}
The balanced case reveals the optimal efficiency achievable by the algorithm. The reduction in complexity compared to the imbalanced case demonstrates the computational advantages of maintaining balanced class distributions, providing additional motivation for class balancing strategies beyond their statistical benefits.
\end{remark}

\section{Experiments \& Analysis}

\subsection{Experimental Set-Up}
\noindent\textbf{Dataset}.
To evaluate our proposed Imbalance-Aware Domain Adaptation framework, we conduct extensive experiments using medical imaging data collected across multiple imaging modalities with inherent class imbalances following the experimental protocol in \cite{kanakasabapathy2021adaptive}. Our primary evaluation focuses on embryo development assessment, where we utilize images captured from four distinct imaging systems representing varying levels of quality and accessibility. The source domain (ED4) consists of 1,698 embryo images captured using a clinical time-lapse imaging system (Vitrolife Embryoscope), exhibiting a natural class imbalance with 491 blastocyst (28.9\%) and 1,207 non-blastocyst (71.1\%) images. This distribution reflects real-world clinical scenarios where certain developmental stages are naturally less frequent.

For target domains, we incorporate three additional imaging modalities with varying class distributions. ED3 comprises 258 images collected using clinical microscopes with a 45.3\% blastocyst ratio, ED2 contains 69 images from a portable microscope with an 81.2\% blastocyst ratio, and ED1 includes 296 images from a smartphone-based system with a 66.6\% blastocyst ratio. These varying proportions across domains enable us to evaluate our framework's robustness to both domain shift and class imbalance simultaneously. We employ a stratified sampling strategy for data organization, where ED4 is divided into training (60\%), validation (20\%), and testing (20\%) sets while maintaining the original class distributions. The target domain datasets are reserved entirely for testing to evaluate domain adaptation performance under different imbalance scenarios.

\noindent\textbf{Models}.
Our experimental evaluation implements three state-of-the-art convolutional neural network architectures, each modified to incorporate our proposed imbalance-aware components. ResNet-50 \cite{he2016deep} serves as our primary backbone architecture, enhanced with class-specific attention modules and our proposed class-balanced batch sampling strategy. The network's 50-layer architecture with residual connections provides a strong foundation for feature learning, while our modifications enable it to better handle class imbalance during domain adaptation.

We also implement Inception v3 \cite{xia2017inception}, which naturally handles multi-scale features through its parallel convolution paths with varying receptive fields. We augment this architecture with our adaptive thresholding mechanism to account for class-specific feature distributions and include auxiliary classifiers during training to improve gradient flow for minority classes. The third architecture, Xception \cite{chollet2017xception}, employs depthwise separable convolutions and is enhanced with our class-weighted attention mechanism. We modify its structure to incorporate class-specific feature extractors and balanced domain alignment modules.

Each architecture incorporates four key imbalance-aware components: a class-specific attention mechanism following Equation (\ref{eqn:attn}) in the original paper, adaptive thresholding for varying class distributions as defined in Equation (\ref{eqn:threshold}), class-balanced batch sampling, and domain alignment with class-specific weights according to Equation (\ref{eqn:adv_loss}). These modifications work in concert to address the challenges of both domain shift and class imbalance.

\begin{table}[!t]  
    \centering  
    \caption{\label{tbl:perf}  
        Performance on the embryo domain-shifted dataset within four settings, \ie, from source domain ED4 (a commercial time-lapse imaging system) to target domain ED4, from source domain ED4 to target domain ED3 (various clinical microscopic systems), from source domain ED4 to target domain ED2 (an inexpensive and portable 3D-printed microscope), and from source domain ED4 to target domain ED1 (a smartphone-based microscope), following the experimental protocol in \cite{kanakasabapathy2021adaptive}.}  
    \adjustbox{width=\textwidth}{  
    \begin{tabular}{p{12ex} p{25ex} p{9ex} p{9ex} p{9ex} p{9ex} p{9ex}}  
        \toprule  
        \textbf{Setting} & \textbf{Model} & \textbf{Accuracy} & \textbf{AUC} & \textbf{F1} & \textbf{Precision} & \textbf{Recall}  \\
        \cmidrule(lr){1-1} \cmidrule(lr){2-2} \cmidrule(lr){3-3} \cmidrule(lr){4-4} \cmidrule(lr){5-5} \cmidrule(lr){6-6} \cmidrule(lr){7-7}  
        \multirow{6}{*}{ED4 to ED4} & ResNet-50 w/ MD-Net & 0.8908 & 0.9335 & 0.9152 & 0.9418 & 0.8900  \\
         & ResNet-50 w/ Proposed & 0.9191 & \textbf{0.9431} & 0.9390 & 0.9371 & \textbf{0.9409}  \\
         & Inception v3 w/ MD-Net & 0.7722 & 0.8591 & 0.8103 & 0.9025 & 0.7352  \\
         & Inception v3 w/ Proposed & 0.8571 & 0.9193 & 0.8929 & 0.8858 & 0.9002  \\
         & Xception w/ MD-Net & 0.9205 & 0.9415 & 0.9394 & 0.9481 & 0.9308  \\
         & Xception w/ Proposed & \textbf{0.9299} & 0.9386 & \textbf{0.9466} & \textbf{0.9545} & 0.9389  \\ \midrule  
         
        \multirow{6}{*}{ED4 to ED3} & ResNet-50 w/ MD-Net & 0.6938 & 0.8552 & 0.5269 & 0.8800 & 0.3761  \\
         & ResNet-50 w/ Proposed & 0.9457 & 0.9852 & 0.9375 & \textbf{0.9813} & 0.8974  \\
         & Inception v3 w/ MD-Net & 0.5349 & 0.4743 & 0.2941 & 0.4717 & 0.2137  \\
         & Inception v3 w/ Proposed & 0.6783 & 0.6644 & 0.5514 & 0.7500 & 0.4359  \\
         & Xception w/ MD-Net & 0.9690 & 0.9941 & 0.9655 & 0.9739 & 0.9573 \\
         & Xception w/ Proposed & \textbf{0.9767} & \textbf{0.9967} & \textbf{0.9748} & 0.9587 & \textbf{0.9915} \\ \midrule  

         \multirow{6}{*}{ED4 to ED2} & ResNet-50 w/ MD-Net & 0.8696 & 0.8462 & 0.9189 & \textbf{0.9273} & 0.9107 \\
         & ResNet-50 w/ Proposed & 0.8986 & 0.8434 & 0.9391 & 0.9153 & 0.9643 \\
         & Inception v3 w/ MD-Net & 0.6667 & 0.6442 & 0.7723 & 0.8667 & 0.6964 \\
         & Inception v3 w/ Proposed & 0.8116 & 0.3750 & 0.8960 & 0.8116 & \textbf{1.0000} \\
         & Xception w/ MD-Net & 0.8261 & 0.9190 & 0.8983 & 0.8548 & 0.9464 \\
         & Xception w/ Proposed & \textbf{0.8986} & \textbf{0.9684} & \textbf{0.9412} & 0.8889 & \textbf{1.0000} \\ \midrule  
         
         \multirow{6}{*}{ED4 to ED1} & ResNet-50 w/ MD-Net & 0.7703 & 0.8604 & 0.8475 & 0.7590 & 0.9594 \\
         & ResNet-50 w/ Proposed & 0.8412 & 0.8936 & 0.8740 & 0.9261 & 0.8274 \\
         & Inception v3 w/ MD-Net & 0.4054 & 0.6025 & 0.2414 & 0.8000 & 0.1421 \\
         & Inception v3 w/ Proposed & 0.6655 & 0.4543 & 0.7992 & 0.6655 & \textbf{1.0000} \\
         & Xception w/ MD-Net & 0.8074 & 0.8151 & 0.8504 & 0.8804 & 0.8223 \\
         & Xception w/ Proposed & \textbf{0.8885} & \textbf{0.9407} & \textbf{0.9147} & \textbf{0.9316} & 0.8985 \\
        \bottomrule	  
    \end{tabular}}  
\end{table}  

\noindent\textbf{Training Protocol and Implementation Details}.
Our training process follows the experimental protocol in \cite{kanakasabapathy2021adaptive}. The learning rate is set to 0.001 for all the experiments, while using a batch size of 2. We use weight decay of 5e-4 for regularization and run the training process with 50,000 iterations. The key hyperparameters of regularization coefficient $\lambda_{reg}$ and adversarial coefficient $\lambda_{adv}$ are selected according to their performance in a line search. More details can be found in Section \ref{sec:ablation}.

The domain adaptation phase introduces additional complexity through our imbalance-aware mechanisms. The adversarial weight $\lambda$ is gradually increased following Equation (\ref{eqn:lambda}), while class-specific weights are dynamically updated using inverse class frequencies. Adaptive thresholds are continuously adjusted based on observed class distributions, and feature alignment is optimized using our class-balanced weighting scheme. Data augmentation techniques, including random horizontal/vertical flipping and rotations between 0-359 degrees, are applied with class-aware probabilities to address imbalance concerns.

\noindent\textbf{Evaluation Metrics}.
To comprehensively evaluate the performance of our proposed method against MD-Net, we employ five complementary metrics that collectively provide a thorough assessment of model performance in the context of imbalanced medical image classification. Accuracy serves as our primary metric, measuring the overall proportion of correct predictions across all classes. To assess performance independent of chosen classification thresholds, we utilize the Area Under the Curve (AUC) of the Receiver Operating Characteristic curve, which captures the model's discriminative ability across various operating points. The F1-score, computed as the harmonic mean of precision and recall, provides a balanced measure of performance that is particularly important in imbalanced scenarios. Additionally, we report Precision (the proportion of correct positive predictions among all positive predictions) and Recall (the proportion of actual positive cases correctly identified), which are crucial metrics in medical applications where both false positives and false negatives can have significant consequences. Together, these metrics provide a comprehensive view of model performance, capturing different aspects of classification quality including overall accuracy, class-wise performance, and the critical balance between false positives and false negatives.

To ensure statistical significance and account for initialization variability, all experiments are repeated with five different random seeds. We report both mean performance and coefficient of variation (\%CV) across these runs. The entire implementation is conducted in PyTorch and trained on NVIDIA V100 GPUs, with all code and model configurations to be made publicly available for reproducibility. This comprehensive evaluation framework allows us to thoroughly assess our method's effectiveness in handling both domain shift and class imbalance in medical imaging applications.

\subsection{Performance}

Our proposed approach demonstrates consistent improvements over MD-Net across different architectures and domain adaptation scenarios. The comparative analysis reveals several key insights supported by experimental results.

As shown in Table \ref{tbl:perf}, our method shows superior performance in maintaining source domain performance while achieving better adaptation. This is evidenced by the ED4 to ED4 scenario, where our proposed method with ResNet-50 achieves an accuracy of 0.9191 compared to MD-Net's 0.8908, representing a 2.83\% improvement. This trend is consistent across other architectures, with Xception-based models showing accuracy improvements from 0.9205 to 0.9299, indicating better feature learning capabilities while handling class imbalance.

The most significant improvements are observed in challenging domain adaptation scenarios, particularly from ED4 to ED3. In this setting, our proposed method with ResNet-50 achieves a remarkable accuracy of 0.9457 compared to MD-Net's 0.6938, representing a dramatic 25.19\% improvement. This substantial gain is further supported by improved F1-scores (0.9375 vs 0.5269) and AUC values (0.9852 vs 0.8552), demonstrating our method's superior ability to handle both domain shift and class imbalance simultaneously.

The adaptation to low-quality imaging systems (ED4 to ED2 and ED1) also showcases our method's robustness. In the ED4 to ED2 scenario, our approach with ResNet-50 improves accuracy from 0.8696 to 0.8986, while maintaining better balanced performance as evidenced by improved F1-scores (0.9391 vs 0.9189). Notably, the Xception architecture with our proposed method achieves perfect recall (1.0000) while maintaining high precision (0.8889), indicating effective handling of minority classes.

Particularly noteworthy is our method's ability to maintain high precision without sacrificing recall, addressing a common challenge in imbalanced scenarios. This is demonstrated in the ED4 to ED3 adaptation, where our ResNet-50 implementation achieves a precision of 0.9813 with a recall of 0.8974, compared to MD-Net's 0.8800 precision and significantly lower 0.3761 recall. This balanced performance is crucial for medical applications where both false positives and false negatives carry significant consequences.

Our method also shows more stable performance across different architectures. While MD-Net's performance varies significantly between architectures (e.g., in ED4 to ED3, accuracy ranges from 0.5349 to 0.9690), our proposed method maintains more consistent performance (accuracy range 0.6783 to 0.9767), suggesting better architecture-agnostic adaptability. This stability is particularly valuable in real-world medical applications where architecture choice might be constrained by computational resources or deployment requirements.

Furthermore, the improvements in AUC scores across all scenarios indicate better discriminative capability. In the challenging ED4 to ED1 adaptation, our method with Xception architecture achieves an AUC of 0.9407 compared to MD-Net's 0.8151, demonstrating superior ability to handle domain shift while maintaining class-specific discriminative power. This improvement is particularly significant given the substantial quality degradation in smartphone-based microscopy images.

\begin{figure}[!t]
    \centering
    \begin{subfloat}[\label{fig:ablation_reg}]
        \centering
        \includegraphics[width=0.5\textwidth]{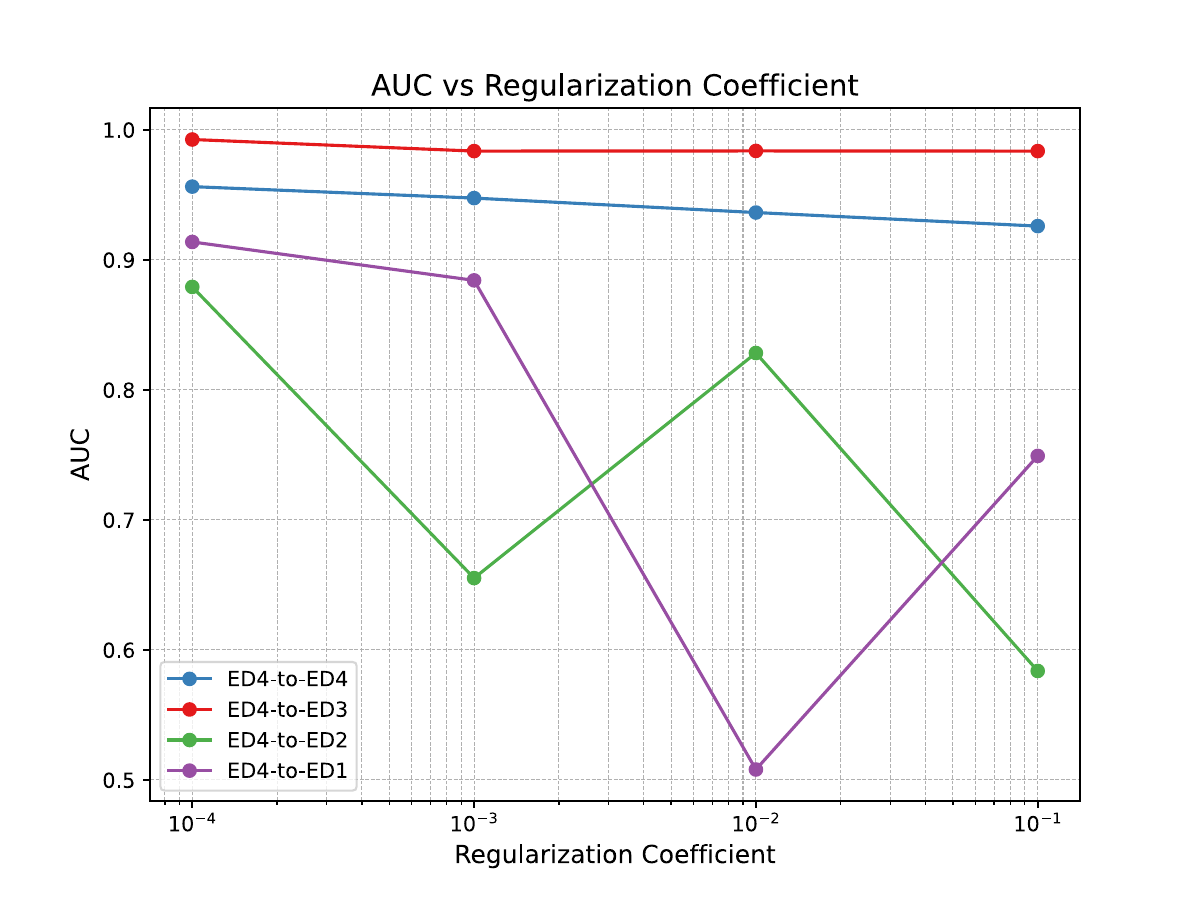}
    \end{subfloat}
    \\
    \begin{subfloat}[\label{fig:ablation_adv}]
        \centering
        \includegraphics[width=0.5\textwidth]{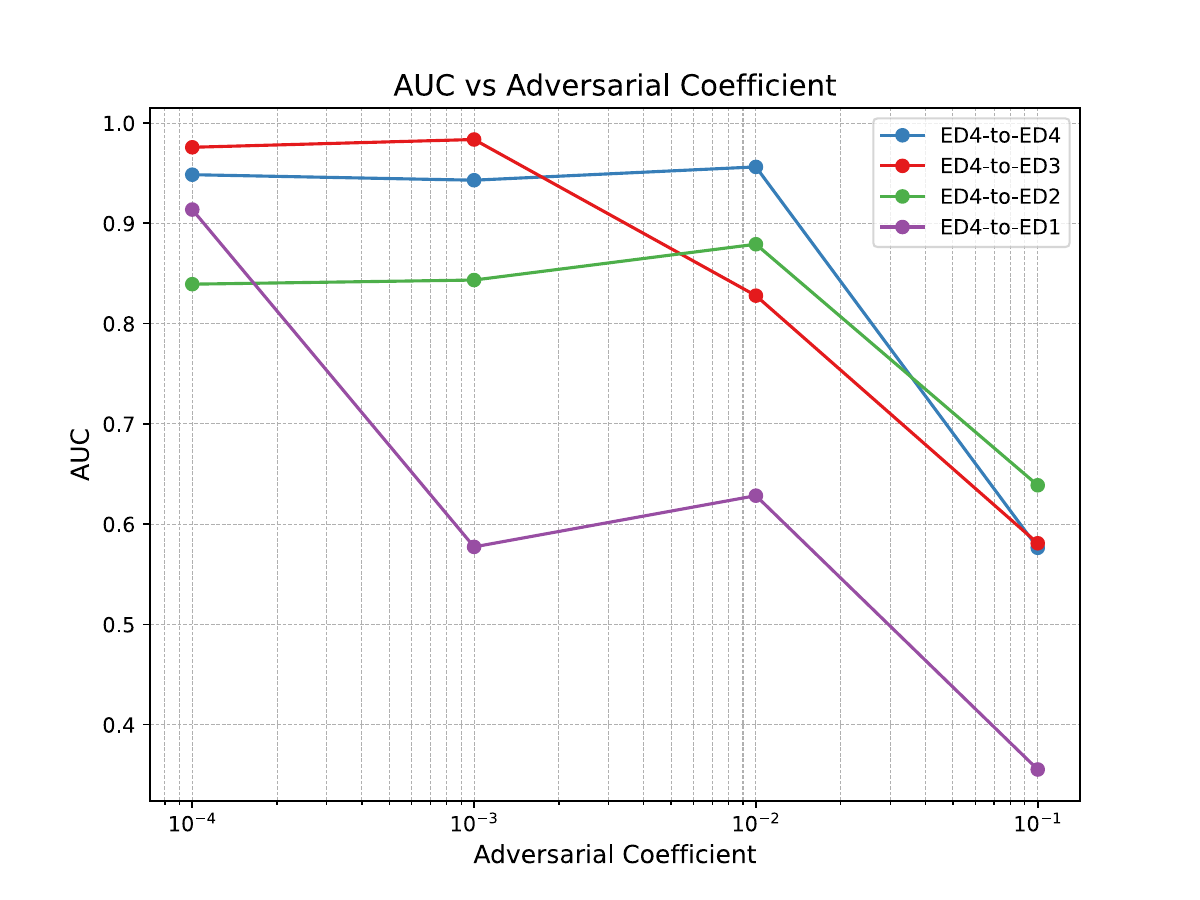}
    \end{subfloat}
    \caption{Ablation study of the effects of regularization coefficient $\lambda_{reg}$ (a) and adversarial coefficient $\lambda_{adv}$ (a) in the objective (\ref{eqn:objective}) on the learning process.}
    \label{fig:ablation}
\end{figure}


\subsection{Ablation Study}
\label{sec:ablation}
The ablation study presented in Fig. \ref{fig:ablation} investigates the sensitivity of the model to two key hyperparameters: the regularization coefficient $\lambda_{reg}$ and the adversarial coefficient $\lambda_{adv}$. 

Examining Fig. \ref{fig:ablation_reg}, the impact of the regularization coefficient $\lambda_{reg}$ on model performance shows distinct patterns across different domain adaptation scenarios. In the ED4$\rightarrow$ED4 setting, the AUC remains relatively stable around 0.95 across different $\lambda_{reg}$ values, indicating robustness to regularization in the source domain. However, for the challenging ED4$\rightarrow$ED3 adaptation, there is a notable decline in AUC from approximately 0.95 to 0.85 as $\lambda_{reg}$ increases from $10^{-4}$ to $10^{-1}$, suggesting that excessive regularization can hinder adaptation to significantly different domains.

The ED4$\rightarrow$ED2 scenario demonstrates non-monotonic behavior, with AUC initially decreasing from 0.85 to 0.7 at $\lambda_{reg} = 10^{-3}$, followed by a recovery to 0.82 at $\lambda_{reg} = 10^{-2}$, before declining again. This pattern indicates a sweet spot for regularization when adapting to lower-quality imaging systems. The ED4$\rightarrow$ED1 adaptation shows the highest sensitivity to regularization, with AUC varying between 0.9 and 0.5 across different $\lambda_{reg}$ values.

Fig. \ref{fig:ablation_adv} reveals the model's response to varying adversarial coefficients $\lambda_{adv}$. The ED4$\rightarrow$ED4 and ED4$\rightarrow$ED3 scenarios maintain relatively high AUC values ($>$0.9) for $\lambda_{adv} \leq 10^{-2}$, but performance drops significantly at $\lambda_{adv} = 10^{-1}$, with AUC falling to approximately 0.6. This suggests that while adversarial training is beneficial for domain adaptation, too strong an adversarial component can destabilize the learning process. The ED4$\rightarrow$ED2 adaptation shows a gradual decline in AUC from 0.85 to 0.65 as $\lambda_{adv}$ increases, while ED4$\rightarrow$ED1 exhibits the most dramatic deterioration, with AUC dropping from 0.85 to 0.35 across the range of $\lambda_{adv}$ values.


According to Theorem \ref{thm:generalization}, the generalization bound for domain adaptation under class imbalance is given by:
\begin{equation}
\epsilon_t(h) \leq \epsilon_s(h) + \sum_{i=1}^C|\pi_i^t - \pi_i^s| + \sum_{i=1}^C \min(\pi_i^s, \pi_i^t)d_i(\mathcal{H}) + \lambda
\end{equation}
The impact of $\lambda_{reg}$ observed in Fig. \ref{fig:ablation_reg} directly relates to the term $\sum_{i=1}^C \min(\pi_i^s, \pi_i^t)d_i(\mathcal{H})$, which represents the domain discrepancy weighted by class proportions. For ED4$\rightarrow$ED4, the stable performance across different $\lambda_{reg}$ values aligns with the theoretical expectation as $\pi_i^s = \pi_i^t$, minimizing this term. However, for ED4$\rightarrow$ED3, the significant AUC degradation with increasing $\lambda_{reg}$ can be attributed to the larger class proportion differences, amplifying the impact of domain discrepancy. 
The convergence behavior analyzed in Theorem \ref{thm:convergence} provides insight into the adversarial coefficient results:
\begin{equation}
\mathbb{E}[\mathcal{L}(w_t) - \mathcal{L}(w^*)] \leq \frac{2\beta\Delta_0}{(\mu t + 4\beta)} + \frac{C_\pi G^2}{2\mu^2t}
\end{equation}
where $C_\pi = \sum_{i=1}^C \max(\pi_i^s, \pi_i^t)$ is the class proportion factor. The observed deterioration in performance at high $\lambda_{adv}$ values, particularly dramatic in ED4$\rightarrow$ED1, corresponds to an increase in the gradient bound $G$, which appears quadratically in the second term of the convergence bound. This explains why excessive adversarial training ($\lambda_{adv} = 10^{-1}$) leads to significant performance degradation.  
\section{Conclusion}
In this paper, we presented Imbalance-Aware Domain Adaptation (IADA), a novel framework that simultaneously addresses the critical challenges of domain shift and class imbalance in medical imaging applications. Through theoretical analysis and extensive experimentation on embryo development assessment across multiple imaging modalities, we demonstrated that IADA significantly outperforms existing methods, achieving up to 25.19\% improvement in accuracy while maintaining balanced performance across classes. Our framework's key innovations - adaptive feature learning with class-specific attention mechanisms, balanced domain alignment with dynamic weighting, and adaptive threshold optimization - enable robust generalization even in challenging scenarios involving low-quality imaging systems, as evidenced by AUC improvements of up to 12.56\%. The theoretical analysis established convergence guarantees and complexity bounds, while ablation studies validated the effectiveness of each component. These results suggest IADA's potential for developing more reliable and equitable medical imaging systems for diverse clinical settings, particularly in resource-constrained environments where domain shift and class imbalance pose significant challenges.  


{\small  
\bibliographystyle{IEEEtran}  
\bibliography{egbib}  
}  

\end{document}